\documentclass[journal]{IEEEtran}

\usepackage{xurl}
\usepackage{cite}
\usepackage{amsmath,amssymb,amsfonts}
\usepackage{graphicx, float}
\usepackage{textcomp,enumitem}
\usepackage{xcolor}
\usepackage{multirow, booktabs, makecell}
\usepackage{algorithm2e}
\RestyleAlgo{ruled}
\usepackage[hidelinks]{hyperref}
\def\BibTeX{{\rm B\kern-.05em{\sc i\kern-.025em b}\kern-.08em
    T\kern-.1667em\lower.7ex\hbox{E}\kern-.125emX}}

\usepackage{amsthm}
\newtheorem{lemma}{\textbf{Lemma}}

\def\rev#1{\textcolor{black}{#1}}
\usepackage{siunitx}

\begin{document}

\title{
Dynamic Load Model for Data Centers with Pattern-Consistent Calibration 
}

\author{
Siyu Lu,~\IEEEmembership{Student Member,~IEEE,}
Chenhan Xiao,~\IEEEmembership{Student Member,~IEEE,}
and~Yang~Weng,~\IEEEmembership{Senior Member,~IEEE}
\vspace{-2em}
\thanks{
S. Lu, C. Xiao, and Y. Weng are with the School of Electrical, Computer and Energy Engineering at Arizona State University, Tempe, AZ 85281 USA, e-mail: \{siyulu12, cxiao20, yang.weng\}@asu.edu.}
}

\maketitle

\begin{abstract}
\rev{The rapid growth of large-scale data centers has made accurate dynamic data center load modeling increasingly important for power system analysis.}
Such loads are characterized by fast workload-driven variability and protection-driven disconnection and reconnection behavior that are not captured by conventional load models.
Existing approaches to data center load modeling include physics-based methods, which provide interpretable structure for grid simulation, and data-driven approaches, which capture empirical workload variability from data. 
However, physics-based models are typically uncalibrated to facility-level operation, while trajectory alignment in data-driven methods often leads to overfitting and unrealistic dynamic behavior.
To resolve these limitations, we design the framework to leverage both physics-based structure and data-driven adaptability.
The physics-based structure is parameterized to enable data-driven pattern-consistent calibration from real operational data, supporting facility-level grid planning. 
We further show that trajectory-level alignment is limited for inherently stochastic data center loads. Therefore, we design the calibration to align temporal and statistical patterns using temporal contrastive learning (TCL).
This calibration is performed locally at the facility, and only calibrated parameters are shared with utilities, \rev{reducing the need to share full-resolution operational time-series traces with utilities}.
The proposed load model is calibrated by real-world operational load data from the MIT Supercloud, ASU Sol, Blue Waters, and ASHRAE datasets. Then it is integrated into the ANDES platform and evaluated on the IEEE 39-bus, NPCC 140-bus, and WECC 179-bus systems.
\rev{We simulate the collective dynamic responses of multiple heterogeneous data centers under system-level disturbances.}
We find that interactions among \rev{data centers} can fundamentally alter post-disturbance recovery behavior, producing compound disconnection-reconnection dynamics and delayed stabilization that are not captured by uncalibrated load models.
\end{abstract}


\vspace{-0.5em}
\section{Introduction}

The rapid growth of large-scale data centers has introduced a new class of grid-side demand, operating at tens to hundreds of megawatts and exhibiting fast variability and protection-driven disconnection behavior~\cite{0929chen2025electricity}. 
Data centers may abruptly disconnect from the grid to safeguard critical server equipment during voltage or frequency disturbances~\cite{NERC_LLTF_2024}.
Such disconnections can shed large volumes of load, creating power imbalances that lead to further voltage and frequency deviations in the grid. 
These deviations will compromise generator synchronism and degrade power quality, increasing the risk of regional stability events, including cascading outages \cite{ERCOT_LEL_RideThrough_2025,shamseldein2025liability}.
Recent events include a 1{,}500~MW simultaneous data center load loss in 2024~\cite{ghosh2025enhancing}, which caused a frequency and voltage rise to 60.047~Hz and 1.07~p.u. 
These events demonstrate that such system-level risks are both practical and urgent.
As data centers continue to scale and cluster geographically~\cite{wan2025grid}, dynamic load modeling of data centers and their grid-side behavior has become essential for reliable power system planning and operation~\cite{NERC2025SOROverview,ERCOT2025LargeLoadStability}.


\rev{Before the expansion of data centers, large-load modeling primarily focused on industrial facilities such as factories and steel mills. These loads were commonly represented by continuous electromechanical dynamics, often combined with static ZIP components to capture voltage-dependent steady-state behavior~\cite{huang2019generic,Arif2018LoadModeling,logar2012modeling}. Aggregate-load self-disconnection and recovery under voltage disturbances have also been extensively studied, including in CIGRE TB~566, which established general representations of disturbance-induced load interruption and restoration~\cite{CIGRE_TB566}. However, these conventional formulations are not tailored to data centers, which are dominated by power-electronic interfaces, UPS-mediated transfer, and rapidly varying stochastic computational workloads~\cite{0929chen2025electricity}.}

To address this issue, recent load modeling approaches developed for data centers can be broadly classified into data-driven and physics-based methods \cite{1021parikh2025strain}.
Data-driven methods learn load behavior directly from historical measurements, using statistical models such as Hidden Markov Models~\cite{HMM}, and learning-based models \cite{Loadprofile} such as recurrent neural network~\cite{kumar2018long}.
While these methods can closely fit training data, their lack of explicit physical structure prevents generalization at a structural level, which makes them unreliable under unseen operating conditions for simulation-based load modeling and grid-side behavior analysis.
To resolve this issue, physics-based load modeling approaches explicitly represent the major subsystems of data centers, including IT workloads, cooling systems, and protection mechanisms, using physically motivated formulations~\cite{0522jimenez2025data,0614peivandizadeh2025theoretical,0822ko2025wide,1006kwon2025operational,1004xie2025enhancing}. For example, IT workload dynamics have been modeled using Poisson processes to capture stochastic computational job arrivals~\cite{0522jimenez2025data}.
These physically interpretable models capture data center load behavior at a structural level, reflecting dominant operating regimes that shape characteristic patterns of load behavior.
However, compared to data-driven approaches, they generally lack the flexibility to adapt load model parameters to the facility-specific operating patterns.
This limits their effectiveness in realistic grid simulations, where heterogeneous data center facilities exhibit distinct load patterns that require facility-specific parameter calibration.

Given these limitations, a natural direction is to develop a data center load modeling framework that not only retains a physics-based backbone but also incorporates a calibrated, data-driven capability to align temporal and statistical load patterns when adapting model parameters to individual facilities.
To this end, this paper formulates data center load modeling as a parameterized protocol, explicitly designed to support subsequent calibration using operational data. 
Building on recent advances that decompose data center load into IT workload, cooling load, auxiliary load, and protection mechanism \cite{0522jimenez2025data, 1004xie2025enhancing}, each subsystem is parameterized according to its underlying physical structure and the characteristic load patterns it induces.
Specifically, IT workload is represented using a duty-idle operating pattern, with server utilization modeled as an Ornstein-Uhlenbeck process~\cite{roberts2016validation}. 
Cooling and auxiliary subsystems are modeled using established motor-based dynamic equations and ZIP representations~\cite{0929chen2025electricity}. Disconnection and reconnection behaviors are captured through a rule-based parameterization that reflects protection logic and operational constraints~\cite{0820conto2025texas}.

Building on this physics-based backbone, we propose a data-driven calibration approach tailored to data center load characteristics for robust parameter estimation with alignment of temporal and statistical load patterns.
Since the data center load is inherently stochastic, directly aligning with instantaneous load can easily lead to overfitting. A more practical strategy is to align the temporal and statistical features of the load, which is sufficient for system operators to reproduce realistic load behavior \cite{0929chen2025electricity}. 
To extract such features, we adopt temporal contrastive learning (TCL)~\cite{liu2024timesurl}, as it provides an unsupervised, window-based representation that discriminates between distinct operating regimes while capturing invariant temporal patterns.
This approach follows directly from the characteristics of data center loads, which exhibit recurring burst-idle windows driven by job scheduling~\cite{0929chen2025electricity}, requiring discrimination among operating regimes and invariance to nuisance variations including small timing shifts and amplitude scaling.
Calibration is then performed by aligning TCL-based feature representations extracted from real facility-level load and simulated load trajectories.

\rev{An additional benefit of the proposed framework is its ability to represent heterogeneity among data center facilities. Different facilities may exhibit distinct workload intensity, burst-arrival behavior, cooling demand, auxiliary-load composition, and protection and recovery settings. Through facility-specific parameter calibration, the proposed model enables multiple data centers with different operating characteristics to be represented within the same power-system simulation.}
\rev{
Meanwhile, the proposed calibration process adopts a limited-disclosure deployment architecture that avoids the direct exchange of full-resolution operational traces. 
As shown in Fig.~\ref{fig:federated}, the utility first distributes an uncalibrated load model to data center facilities. Each facility then performs pattern-consistent calibration locally and returns only calibrated model parameters.
Throughout this process, raw operational load data, particularly workload traces, are never shared. 
It reduces the direct exposure of proprietary operational information associated with computing utilization, scheduling patterns, and business activity~\cite{WRI2025DataCenters}.
}

\begin{figure}[h]
    \centering
     \vskip -0.15in
    \includegraphics[width=0.9\linewidth]{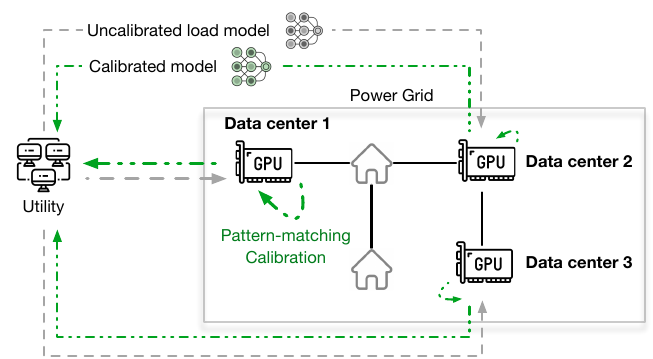}  
    \vskip -0.1in
    \caption{\rev{Pattern-consistent calibration framework for realistic data center dynamic behavior without direct exchange of full-resolution operational traces.}}
    \label{fig:federated}
     \vskip -0.1in
\end{figure}

To evaluate dynamic responses under system-level disturbances, we design validation studies to capture variations in network scale, geographical grid characteristics, and heterogeneous data center deployment. Accordingly, simulations are conducted on IEEE 39-bus, NPCC 140-bus, and WECC 179-bus systems using \texttt{ANDES}~\cite{ANDES}. 
To enhance realism, model parameters are calibrated using operational data from both academic and industrial computing facilities, including MIT Supercloud~\cite{MIT-data}, ASU Sol, Blue Waters~\cite{BlueWatersData}, and ASHRAE~\cite{ASHRAE_kaggle}, with protection and reconnection parameters derived from utility reports~\cite{ERCOT2025LargeLoadStability}. 
The results demonstrate that parameter calibration is critical for accurately capturing the system-level impacts of data centers in dynamic grid simulations.

\rev{
Our main contributions are summarized as follows:
\begin{itemize}
    \item A parameterized, physics-based dynamic data center load model is developed to jointly represent workload variability, cooling and auxiliary load, and protection-driven disconnection and reconnection.
    \item A TCL-based pattern-consistent calibration method is proposed to estimate facility-specific model parameters by aligning the temporal and statistical patterns of measured and simulated load trajectories.
    \item Real-data validation and \texttt{ANDES}-based studies demonstrate that facility-specific calibration is essential for capturing compound disconnection--reconnection dynamics, delayed recovery, and system risks that increase with data center penetration.
\end{itemize}
}

The remainder of the paper is organized as follows. Section~\ref{sec:lel-modeling} presents the physics-based, parameterized load modeling backbone. Section~\ref{sec:pattern-match} introduces the pattern-consistent calibration framework. Section~\ref{sec:numerical} evaluates mixed-data center simulations using \texttt{ANDES}. Section~\ref{sec:conclusion} concludes the paper.

\vspace{-0.5em}
\section{Parameterized Physics-Based Load Modeling of Data Centers}
\label{sec:lel-modeling}

To enable efficient calibration under heterogeneous and facility-specific operational data, we propose a physics-based load modeling backbone with explicit parameterization. 
It enables calibration by tuning a small set of physically meaningful parameters, rather than learning unconstrained input-output mappings as in pure data-driven load modeling \cite{HMM,Loadprofile,kumar2018long}.
Following established data center modeling practices \cite{0522jimenez2025data,0614peivandizadeh2025theoretical,0822ko2025wide,1006kwon2025operational,1004xie2025enhancing}, the aggregate load is decomposed into IT workload (Section~\ref{sec:duty-idle}), cooling load and auxiliary load (Section~\ref{sec:cooling}). A parameterized protection mechanism is further introduced to describe disconnection and reconnection behavior (Section~\ref{sec:UPS}).
Together, these components capture the essential dynamics of data center power consumption.

\vspace{-1em}
\subsection{Duty-Idle Workload Modeling}
\label{sec:duty-idle}

In data centers, IT workload power consumption dominates the total electrical demand, typically accounting for more than half of total demand~\cite{OffuttZhu2025_R48646}.
Accordingly, IT workload is treated as the core subsystem in the proposed load model.
\rev{
The objective of this subsection is not to reproduce individual task scheduling, queueing, or server-level execution, which are usually unavailable to system operators \cite{0929chen2025electricity}.
Instead, we model the aggregate IT utilization observed at the facility level, which determines the grid-side active-power demand of the IT subsystem.
}
Within IT workloads, CPU and GPU utilization can rapidly transition between idle and peak duty in response to computational job arrivals, as observed in MIT Supercloud GPU utilization traces~\cite{MIT-data}. 
Accordingly, a duty-idle formulation~\cite{0522jimenez2025data} is adopted for workload modeling:
\begin{equation}\label{eq:pwork}
p_{\text{work}}(t) = p_{\text{base}} + \eta(t)\,\big(p_{\text{full}}-p_{\text{base}}\big), 
\quad \eta(t)\in[0,1],
\end{equation}
where $p_{\mathrm{base}}$ denotes the standby or idle power draw, $p_{\mathrm{full}}$ is the rated full-duty power, and $\eta(t)$ denotes the \rev{aggregate} utilization level. In Eq. \eqref{eq:pwork}, accurately modeling $\eta(t)$ is central to capturing workload bursts, sustained high-utilization periods, and return-to-idle behavior in data centers. 

\rev{
To model $\eta(t)$, prior data-center load models have used Poisson-based processes to represent stochastic computational job arrivals~\cite{0522jimenez2025data}.
Such a representation is useful for modeling the random timing of job-arrival events, but a pure Poisson impulse process is memoryless and does not by itself define a persistent utilization state.
Consequently, when used in isolation, it cannot capture the temporal correlation, finite workload duration, and gradual return-to-idle behavior observed in real CPU/GPU utilization traces \cite{MIT-data}.
To preserve the useful arrival-time interpretation of the Poisson process while introducing aggregate workload memory, we model $\eta(t)$ as a jump-driven Ornstein--Uhlenbeck (OU) utilization process~\cite{roberts2016validation}:
}
\rev{
\begin{equation}\label{eq:eta}
d\eta(t)
=
-\frac{1}{\tau_\eta}\big(\eta(t)-\mu_\eta\big)\,dt
+\sigma_\xi\,dW_t
+\sum_{k} A_k\,\delta(t-t_k).
\vspace{-0.3em}
\end{equation}
The physical meaning of Eq.~\eqref{eq:eta} is as follows.
The mean-reversion term $-\frac{1}{\tau_\eta}\big(\eta(t)-\mu_\eta\big)\,dt$ drives the utilization toward the nominal operating level $\mu_\eta\in[0,1]$ and captures the fact that, after a workload burst, aggregate utilization tends to relax as jobs finish and resources return to normal operation.
The mean-reversion time constant $\tau_\eta>0$ determines workload persistence: larger $\tau_\eta$ corresponds to longer sustained high-utilization periods.
The disturbance term $\sigma_\xi dW_t$ captures small background variability from fine-scale task scheduling, measurement noise, and unmodeled fluctuations. $W_t$ is a standard Wiener process, and $\sigma_\xi$ controls the intensity of such workload fluctuations. 
Then, the jump term $\sum_{k} A_k\,\delta(t-t_k)$ captures abrupt workload increases caused by burst job arrivals.
The job arrival times $\{t_k\}$ are modeled by a Poisson process with rate $\lambda$, and the jump magnitudes $A_k$ are drawn from a log-normal distribution with parameters $\theta_A$ to represent positive and heavy-tailed burst sizes.
}

In numerical simulation, Eq.~\eqref{eq:eta} is implemented using a bounded discrete-time update rule at each $m$-th time step:
\begin{equation}\label{eq:eta-discrete}
\eta_{m+1}
=
\Pi_{[0,1]}(
\eta_m
+
\frac{\Delta t}{\tau_\eta}\big(\mu_\eta-\eta_m\big)
+
\sigma_\xi\sqrt{\Delta t}\epsilon_m
+
\sum_{\ell=1}^{N_m} A_{\ell,m}
),
\end{equation}
where $\Delta t$ is the time interval, $\epsilon_m\sim\mathcal{N}(0,1)$, $N_m\sim\mathrm{Poisson}(\lambda\Delta t)$, and $\Pi_{[0,1]}(\cdot)$ clips the utilization to the physically feasible interval $[0,1]$.
Eq.~\eqref{eq:eta-discrete} also defines the treatment of overlapping jobs. New job arrivals are modeled as additive increments to the current aggregate utilization state, rather than as resets of the workload process. Thus, unresolved jobs are implicitly retained through the state variable $\eta_m$, while their persistence is governed by the mean-reversion time constant $\tau_\eta$. Multiple arrivals within one simulation interval are accumulated through $\sum_{\ell=1}^{N_m}A_{\ell,m}$, and the projection operator $\Pi_{[0,1]}(\cdot)$ enforces the physical utilization bound.

\rev{
Compared with a pure Poisson workload process, the proposed OU-based formulation provides a more realistic representation of aggregate IT utilization. In a Poisson-based workload model, the utilization can be represented as $\eta_{\mathrm{Pois}}(t)=\Pi_{[0,1]}[\mu_\eta+\sum_k A_k \mathbf{1}_{t_k\le t<t_k+D_k}]$, where $t_k$ is the Poisson arrival time and $D_k$ is the finite job holding duration. Such a model directly maps arrivals to finite workload jobs, but lacks an explicit mechanism for temporal correlation, gradual relaxation, and background utilization variability. In contrast, our model uses Poisson arrivals only to determine the timing of burst-triggering events, while the OU dynamics in Eq.~\eqref{eq:eta} transform these arrivals into a persistent aggregate utilization trajectory through mean reversion toward $\mu_\eta$, small continuous fluctuations, and the physical bound $\eta(t)\in[0,1]$. As illustrated in Fig.~\ref{fig:OU}, the proposed formulation therefore captures both abrupt workload bursts and sustained utilization dynamics, making it more suitable for grid-side dynamic load modeling than a pure Poisson workload representation.
The Poisson-arrival assumption is therefore used only as a first-order model for burst-triggering events. As empirically examined in Section~\ref{sec:validation-calibration}, this assumption is most appropriate for GPU workload traces whose arrival statistics are approximately stationary within the calibration window. For aggregate loads that exhibit stronger periodicity, scheduled operation, or abrupt regime changes~\cite{mishradiscerning}, the proposed OU framework can be generalized by replacing stationary parameters with time-varying or regime-dependent quantities, such as $\mu_\eta(t)$, $\tau_\eta(t)$, and $\lambda(t)$.
}

\begin{figure}[H]
    \centering
    \vskip -0.15in
    \includegraphics[width=1\linewidth]{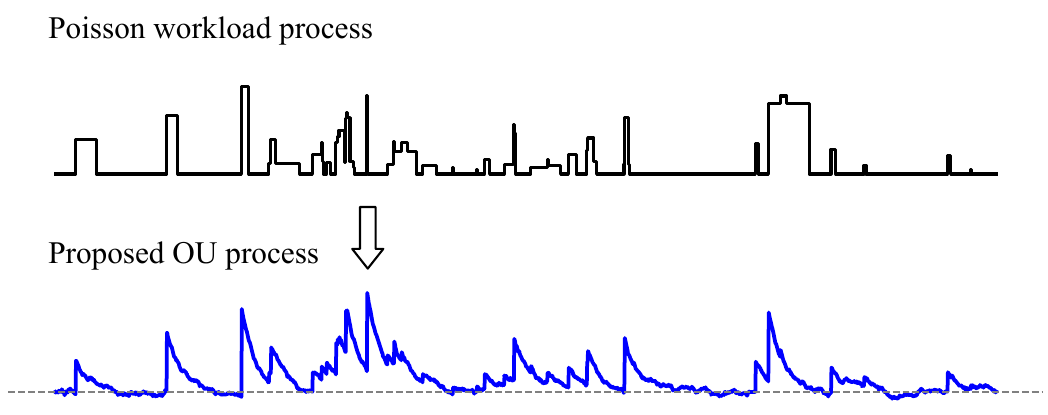}
    \vskip -0.1in
    \caption{\rev{Comparison between a Poisson-based workload process and the proposed OU workload process using the same sampled job-arrival events.}
    }
    \label{fig:OU}
    \vskip -0.1in
\end{figure}

\rev{
Compared with existing workload models, including first-order filtered models that introduce relaxation to avoid memoryless impulse behavior~\cite{0522jimenez2025data}, the proposed OU formulation in Eq.~\eqref{eq:eta} directly links $\mu_\eta$, $\tau_\eta$, $\sigma_\xi$, $\lambda$, and $\theta_A$ to observed workload patterns (see Section~\ref{sec:validation-calibration}), including nominal utilization, temporal persistence, background variability, burst-arrival rate, and burst magnitude. These physical interpretations provide a compact parameterization of IT workload dynamics and enable facility-specific adaptation through the calibration framework in Section~\ref{sec:pattern-match}.
}

To support the pattern-consistent calibration framework introduced in Section~\ref{sec:pattern-match}, the key workload parameters are
\begin{equation}\label{eq:workload-parameter}
    \Theta_{\mathrm{work}} = \{
    p_{\mathrm{base}}, p_{\mathrm{full}}, \tau_\eta, \mu_\eta, \sigma_\xi, \lambda, \theta_A
    \}.
\end{equation}

\subsection{Cooling and Auxiliary Load Modeling}
\label{sec:cooling}

Besides IT workload, cooling systems account for \rev{a substantial share of} the remaining power demand~\cite{OffuttZhu2025_R48646} for data centers.
Since cooling units (e.g., chillers, compressors, fans, and pumps) are predominantly motor-driven, the aggregated cooling load can be represented using the Western Electricity Coordinating Council (WECC) Composite Load Model for Dynamic Simulations (CMPLDW) framework~\cite{kosterev2008load}. 
Within this framework, the dynamic behavior of the cooling motors is modeled in the synchronous rotating direct-quadrature reference frame ($dq$-frame), and the instantaneous active and reactive power consumption is computed as
\begin{equation}
\begin{aligned}
p_{\mathrm{cool}}(t) &= v_{ds}(t)\,i_{ds}(t) + v_{qs}(t)\,i_{qs}(t),\\
q_{\mathrm{cool}}(t) &= v_{qs}(t)\,i_{ds}(t) - v_{ds}(t)\,i_{qs}(t),
\end{aligned}
\end{equation}
where the stator voltages $(v_{ds}, v_{qs})$ are obtained from the terminal bus voltage, and the stator currents $(i_{ds}, i_{qs})$ arise from solving the coupled differential-algebraic equations governing stator and rotor fluxes, slip, and electromagnetic torque. These equations follow the standard WECC composite load motor formulation, parameterized by the electrical constants $(R_s,X_s,X_m,R_r,X_r)$ and the mechanical inertia $H_m$~\cite{kosterev2008load}.

Furthermore, cooling motors exhibit voltage-sensitive behavior during disturbances. Under deep voltage sags, reduced electromagnetic torque and increased reactive current demand accelerate rotor slip. Stalling occurs if $|V| < V_{\mathrm{stall}}$ persists longer than $\tau_{\mathrm{stall}}$. 
When stalling is triggered, the cooling block remains disconnected for a recovery interval $T_{\mathrm{cool}}$. 
To support the pattern-consistent calibration framework, we parameterize the cooling module by
\begin{equation}\label{cooling-parameter}
    \Theta_{\mathrm{cool}} = \{R_s,X_s,X_m,R_r,X_r,H_m,V_{\mathrm{stall}},\tau_{\mathrm{stall}},T_{\mathrm{cool}}
    \}.
\end{equation}

\rev{
Because aggregate facility measurements may not uniquely identify individual motor parameters, $\Theta_{\mathrm{cool}}$ is interpreted as a facility-level dynamic-equivalent parameter set rather than component nameplate values. As described in Section~\ref{sec:pattern-match}, it is calibrated against measured cooling-load behavior subject to physically admissible parameter bounds, which regularize the inverse problem and exclude nonphysical solutions. The parametric WECC formulation is preferred over a non-parametric alternative because it preserves interpretable voltage--current coupling, reactive-power response, rotor-slip dynamics, and stalling behavior under unseen disturbances without requiring extensive disturbance-labeled data.
}

The remaining data center power consumption is attributed to non-IT auxiliary loads, including lighting, control units, security electronics, networking equipment, and other miscellaneous components \cite{0522jimenez2025data}.
Due to their fast electrical response and negligible internal dynamics compared to cooling motors, these loads are modeled using a standard ZIP formulation~\cite{0522jimenez2025data}:
\begin{equation}
\begin{aligned}
p_{\mathrm{aux}}(V)
&= p_{\mathrm{aux},0}\cdot (\alpha_Z (\tfrac{V}{V_0})^2
+\alpha_I (\tfrac{V}{V_0})
+\alpha_P),\\
q_{\mathrm{aux}}(V)
&= \beta_{\mathrm{aux}}\cdot p_{\mathrm{aux}}(V),
\end{aligned}
\end{equation}
where $p_{\mathrm{aux},0}$ denotes the nominal auxiliary active power at the reference voltage $V_0$, $\beta_{\mathrm{aux}}$ is a constant power factor, and the coefficients satisfy $\alpha_Z + \alpha_I + \alpha_P = 1$. To support facility-specific calibration, the auxiliary module is parameterized by
\begin{equation}\label{aux-parameter}
    \Theta_{\mathrm{aux}} = \{
    p_{\mathrm{aux},0}, \alpha_Z,\alpha_I, \alpha_P, \beta_{\mathrm{aux}}
    \}.
\end{equation}

\subsection{Grid-Interfacing Protection and Recovery Modeling}
\label{sec:UPS}

\rev{
Building on established aggregate-load self-disconnection and recovery formulations, including those summarized in CIGRE TB~566~\cite{CIGRE_TB566}, we develop a data-center-specific protection and recovery module. The module adapts these established mechanisms to represent UPS-mediated transfer of data center demand from grid supply to local supply, followed by staged grid reconnection.
}
In practice, data centers deploy multi-stage uninterruptible power supply systems (UPS)~\cite{NERC_LLTF_2024,0929chen2025electricity} to isolate IT loads from the grid and provide local supply.
Such protection is activated when voltage or frequency deviations exceed prescribed tolerance thresholds:
\begin{equation}\label{eq:beyond-normal}
|V - V_{\mathrm{ref}}| > \Delta V 
    \text{ or } |\omega - \omega_{\mathrm{ref}}| > \Delta \omega,
\end{equation}
where $\Delta V$ and $\Delta \omega$ define the allowable deviations from the nominal voltage $V_{\mathrm{ref}}$ and frequency $\omega_{\mathrm{ref}}$. 
When protection is activated, partial or complete load shedding occurs and is modeled as
\begin{equation}\label{eq:protect}
\tilde{p}_{\mathrm{load}}(t) = \kappa(t)  p_{\mathrm{load}}(t),
\quad
\tilde{q}_{\mathrm{load}}(t) = \kappa(t)  q_{\mathrm{load}}(t),
\end{equation}
where 
\(p_{\mathrm{load}} = p_{\mathrm{work}}+p_{\mathrm{cool}}+p_{\mathrm{aux}}\)
denotes the aggregate pre-protection power demand, and
\(q_{\mathrm{load}} = q_{\mathrm{cool}}+q_{\mathrm{aux}}\)
denotes the corresponding reactive demand since the IT workload is assumed to operate at near-unity power factor.
The term $\kappa(t)\in[0,1]$ represents a retained-load fraction that captures the combined impact of UPS isolation for IT loads and protection-induced disconnection of cooling and auxiliary components.

To prevent nuisance responses to shallow or short-lived disturbances, protection actions are initiated only if the violation in Eq. (\ref{eq:beyond-normal}) persists for longer than a minimum duration $t_{\mathrm{delay}}^{\mathrm{trip}}$, after which load shedding is triggered according to the retained-load formulation in Eq. (\ref{eq:protect}).

After load shedding is triggered, recovery-driven reconnection is governed by both electrical and operational constraints. Under electrical constraints, the bus voltage and system frequency must return to acceptable bounds, $|V - V_{\mathrm{ref}}| \le \Delta V $ and $ |\omega - \omega_{\mathrm{ref}}| \le \Delta \omega$ and remain stable for a waiting interval $t_{\mathrm{wait}}^{\mathrm{recon}}$.
Under operational constraints, a minimum reconnection delay $t_{\mathrm{delay}}^{\mathrm{recon}}$ must elapse following disconnection, representing requirements such as UPS recharge or controller reset time.
Once these conditions are met, the retained-load fraction $\kappa(t)$, constrained by $\kappa_{\min} \le \kappa(t) \le \kappa_{\max}$, is gradually restored toward unity under a ramp-rate constraint $\dot{\kappa}(t) \le r_{\kappa}$. The parameter $r_{\kappa}$ denotes a site-dependent reconnection rate that prevents unrealistic step changes in demand and captures staged restoration behavior.

To support the pattern-consistent calibration framework, we parameterize the protection and recovery module by
\begin{equation}\label{eq:prot-parameter}
\begin{aligned}
\Theta_{\mathrm{prot}} = \{
V_{\mathrm{ref}},
\omega_{\mathrm{ref}},
\Delta V,
\Delta \omega,
t_{\mathrm{delay}}^{\mathrm{trip}},
t_{\mathrm{wait}}^{\mathrm{recon}},
t_{\mathrm{delay}}^{\mathrm{recon}},\\
\kappa_{\min},
\kappa_{\max}, 
r_{\kappa} 
\}.
\end{aligned}
\end{equation}

\section{Pattern-Consistent Calibration Framework}
\label{sec:pattern-match}

With the physics-based load model established and key parameters identified, we develop a pattern-consistent calibration framework to adapt the model to facility-specific data.

\subsection{Statistical Pattern-Consistent Calibration}
\label{sec:pattern-workload}

For the workload, cooling, and auxiliary subsystems, we calibrate the parameter sets
\((\Theta_{\mathrm{work}}, \Theta_{\mathrm{cool}}, \Theta_{\mathrm{aux}})\)
by maintaining consistency in representative statistical patterns between real operational loads and simulated load trajectories, rather than by directly aligning the time-series trajectory.
\rev{This choice is motivated by the fact that data center loads are unlabeled and inherently stochastic.} In such cases, point-wise trajectory alignment is prone to overfitting and can lead to unrealistic load behavior during load simulation. We formalize this limitation in Lemma~\ref{lem:impulse_nonident} using a simplified workload model in Eq.~\eqref{eq:eta}.

\begin{lemma}[Overfitting risk of trajectory alignment for data center workloads]
\label{lem:impulse_nonident}
Consider the core impulse-driven component of a data center workload modeling in Eq.~(\ref{eq:eta}):
\(
I(t)=\sum_{k}\delta(t-t_k)
\)
over a finite time horizon \([0,T]\), where
\(\{t_k\}\) are the job-arrival times of a homogeneous Poisson process with rate \(\lambda^{\ast}>0\).
A single time-series sequence of the impulse \(I(t)\) over \([0,T]\) does not uniquely identify the generating rate \(\lambda^{\ast}\).
\end{lemma}

\begin{proof}
Under any homogeneous Poisson process with rate \(\lambda>0\), the likelihood of observing a time-series sequence \(I(t)\) over \([0,T]\) with job-arrival times \(0<t_1<\cdots<t_n<T\) is:
\(
\mathbb{P}(N(T)=n)=e^{-\lambda T}(\lambda T)^n/n!,
\)
which is strictly positive.
Hence infinitely many values of \(\lambda\) are compatible with the same observed sequence, and the generating rate cannot be uniquely identified from a single training sequence of \(I(t)\).
\end{proof}

Lemma~\ref{lem:impulse_nonident} indicates that trajectory-level alignment constrains only a single realization and is therefore insufficient to recover the underlying stochastic dynamics of data center workloads. \rev{This motivates calibration based on statistical patterns rather than point-wise matching.} However, the relevant patterns extend beyond the marginal distribution of measured power values. That is because data center loads exhibit strong temporal locality and regime-switching behavior \cite{0929chen2025electricity}, including transitions among burst, idle, and sustained operating regimes. A suitable representation must therefore capture local temporal structure while remaining robust to small timing and measurement perturbations.

\rev{Conventional moment matching applied directly to raw power measurements can reproduce marginal characteristics such as the mean and variance, but it does not explicitly preserve temporal ordering \cite{fulcher2013highly}. Thus, two load traces may have similar marginal moments while differing substantially in burst frequency, duration, persistence, ramping behavior, and regime transitions. This motivates a data-driven window-level representation that can be learned from unlabeled traces, distinguish different operating patterns, and remain insensitive to perturbations that do not alter the underlying regime.}

\rev{To satisfy these requirements, we adopt temporal contrastive learning (TCL) \cite{liu2024timesurl}. TCL treats two perturbed views of the same operational window as a positive pair and windows from different temporal contexts as negative pairs. The former promotes robustness to nuisance perturbations, while the latter encourages the encoder to preserve distinctions among burst, idle, and sustained-load behaviors. Unlike direct trajectory reconstruction, TCL does not require reproduction of the exact sample-level realization of a stochastic trace and avoids reliance on a manually selected set of temporal statistics whose relevance may vary across data centers and subsystems. It therefore provides a common feature space in which measured and simulated trajectories can be compared according to their local temporal patterns.}

\rev{
Based on the above considerations, TCL is implemented through window construction, stochastic augmentation, and contrastive representation learning. Let \(\{\mathbf{x}(t)\}_{t=1}^{T}\) denote the time-series load signal associated with a given subsystem, comprising active and reactive power measurements. The signal is first segmented into $N$ overlapping windows, where the $i$-th window is defined as \(\mathbf{x}^{(i)}={\mathbf{x}(t)}_{t=t_i}^{t_i+L-1}, i=1,2,\ldots,N\), with window length \(L\), stride \(H<L\), starting index \(t_i=1+(i-1)H\), and total number of windows \(N=\left\lfloor\frac{T-L}{H}\right\rfloor+1\).
For each window \(\mathbf{x}^{(i)}\), two stochastic augmentations, \(\widetilde{\mathbf{x}}_{1}^{(i)}\) and \(\widetilde{\mathbf{x}}_{2}^{(i)}\), are generated through random amplitude scaling and additive noise. The augmentation ranges are selected to represent plausible measurement and operating perturbations without removing physically meaningful load-magnitude information. The two augmented views of the same window form a positive pair, while augmented views associated with different windows form negative pairs. An encoder network \(f_{\phi}\), parameterized by \(\phi\), maps each augmented window to a representation \(\mathbf{z}_{k}^{(i)}\in\mathbb{R}^{d}\):
\(
\mathbf{z}_{k}^{(i)}
=
f_{\phi}\!\left(\widetilde{\mathbf{x}}_{k}^{(i)}\right),
\ 
i=1,\ldots,N,\  k\in\{1,2\}.
\)
The encoder is trained by minimizing the normalized temperature-scaled contrastive loss
\begin{equation}\label{eq:TCL-loss}
\mathcal{L}_{\mathrm{TCL}}
=
-\frac{1}{N}
\sum_{i=1}^{N}
\log
\frac{
\exp\!\left(
\operatorname{sim}(\mathbf{z}_{1}^{(i)},\mathbf{z}_{2}^{(i)})/\tau
\right)
}{
\sum_{j=1}^{N}
\exp\!\left(
\operatorname{sim}(\mathbf{z}_{1}^{(i)},\mathbf{z}_{2}^{(j)})/\tau
\right)
},
\end{equation}
where \(\operatorname{sim}(\cdot,\cdot)\) denotes cosine similarity and \(\tau\) is the temperature parameter. Once training is completed, the encoder parameters \(\phi\) are frozen and are not updated during physical-parameter calibration.
}

\rev{The frozen encoder is then applied to all unaugmented windows \(\mathbf{x}^{(i)}\), yielding the representations \(\{\mathbf{z}^{(i)}\}_{i=1}^{N}\). Each \(\mathbf{z}^{(i)}\) captures the representative temporal pattern of the corresponding window. To aggregate the window-level information into a compact descriptor of the complete load trajectory, we compute the first- and second-order elementwise moments of the representation set:
\begin{equation}\label{eq:TCL-feature}
\mathbf{s}
=
\begin{bmatrix}
\boldsymbol{\mu}_{z}\\
\boldsymbol{\nu}_{z}
\end{bmatrix}
=
\begin{bmatrix}
\frac{1}{N}\sum_{i=1}^{N}\mathbf{z}^{(i)}\\
\frac{1}{N}\sum_{i=1}^{N}
\left(
\mathbf{z}^{(i)}
-
\frac{1}{N}\sum_{j=1}^{N}\mathbf{z}^{(j)}
\right)^{\odot 2}
\end{bmatrix}
\in\mathbb{R}^{2d},
\end{equation}
where \((\cdot)^{\odot 2}\) denotes elementwise squaring.
}

\rev{
The first two latent moments are used as compact descriptors of the embedding distribution: \(\boldsymbol{\mu}_{z}\) captures the centroid and prevalence-weighted composition of operating regimes, while \(\boldsymbol{\nu}_{z}\) measures heterogeneity across idle, burst, and sustained-load windows. Because each embedding represents an entire time window, these moments summarize TCL-extracted temporal information rather than only marginal sample statistics. Higher-order moments or full covariance matching could retain additional information but would increase the objective dimension and require more data for reliable estimation.
}

\rev{
For a given subsystem, let \(\{\mathbf{x}_{\mathrm{data}}(t)\}_{t=1}^{T}\) denote the real operational data and let \(\{\mathbf{x}_{\mathrm{model}}(t;\Theta)\}_{t=1}^{T}\) denote the simulated trajectory generated using the physical parameter vector
\(
\Theta\in
\left\{
\Theta_{\mathrm{work}},
\Theta_{\mathrm{cool}},
\Theta_{\mathrm{aux}}
\right\}.
\)
The corresponding statistical pattern vectors, \(\mathbf{s}_{\mathrm{data}}\) and \(\mathbf{s}_{\mathrm{model}}(\Theta)\), are obtained from Eq.~\eqref{eq:TCL-feature} using the same frozen encoder. The physical-parameter calibration is then performed by solving
\begin{equation}\label{eq:TCL-matching}
\Theta^{*}
=
\arg\min_{\Theta\in\Omega}
\left\|
\mathbf{s}_{\mathrm{data}}
-
\mathbf{s}_{\mathrm{model}}(\Theta)
\right\|_{2}^{2},
\end{equation}
where \(\Omega=[\Theta^{\min},\Theta^{\max}]\) denotes the physically admissible parameter region. The parameter bounds, numerical solver, convergence criteria, and computational cost are reported in Section~\ref{sec:numerical-setup}.
The resulting $\Theta^{*}$ is then used in the corresponding subsystem model to generate trajectories whose latent pattern statistics best match those of the operational data. 
}

\subsection{Rule-Based Specification of Protection Parameters}
\label{sec:pattern-protection}

Calibration of protection-related parameters $\Theta_{\mathrm{prot}}$ differs fundamentally from that of workload, cooling, and auxiliary subsystems. 
Protection-driven disconnection and recovery-driven reconnection events are triggered infrequently, and are often absent from routine operational data~\cite{shamseldein2025liability}. 
As a result, attempting to estimate protection parameters purely from data is unreliable and can lead to unstable or nonphysical calibration results. 
Instead, these parameters are specified directly from facility protection settings, which are explicitly defined in practice and reflect established operational constraints while avoiding dependence on rare-event observations.

We therefore adopt a rule-based calibration approach for the protection parameter set \(\Theta_{\mathrm{prot}}\). 
Specifically, each data center facility provides \(\Theta^{\ast}_{\mathrm{prot}}\) directly through a standardized disclosure form that specifies its voltage and frequency trip thresholds, enforced time delays, throttling and reconnection ramp rates, and reconnection limits. 
These parameters encode the facility’s internal protection settings and operational constraints without exposing raw operational data.

\begin{figure}[h]
    \centering
    \vskip -0.1in
    \includegraphics[width=1\linewidth]{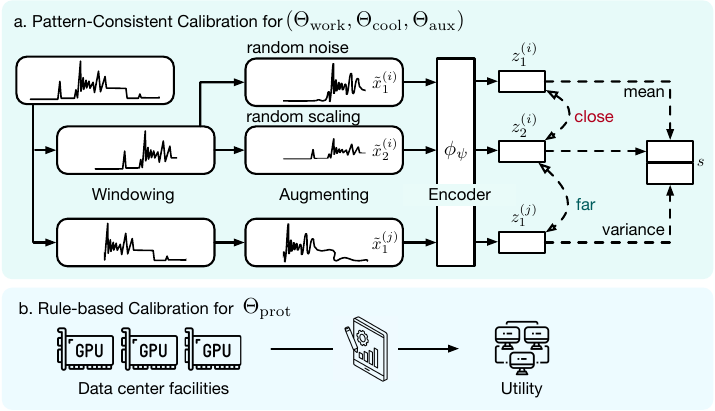}
    \vskip -0.1in
    \caption{Pattern-consistent calibration of data center modeling parameters.}
    \label{fig:pattern-match}
\end{figure}

The overall calibration workflow is illustrated in Fig.~\ref{fig:pattern-match}. 
The resulting calibrated parameter set
$\{\Theta^{\ast}_{\mathrm{work}},
\Theta^{\ast}_{\mathrm{cool}},
\Theta^{\ast}_{\mathrm{aux}},
\Theta^{\ast}_{\mathrm{prot}}\}$
ensures that the load model reproduces the empirical statistical structure observed in real load data, which is essential for accurate grid simulations in practice. 
After calibration, system-level assessment of data center impacts becomes feasible, as each data center site can independently obtain a calibrated load model. 
Section~\ref{sec:numerical} integrates multiple calibrated data center instances into the \texttt{ANDES}~\cite{ANDES} simulation platform to evaluate their collective effects on transient stability, disturbance propagation, and post-disturbance recovery behavior.

\section{Numerical Results}
\label{sec:numerical}

This section assesses whether the proposed calibrated load models reproduce realistic data center behavior under system contingencies.

\subsection{Experiment Setup}
\label{sec:numerical-setup}
The calibrated model is evaluated using real-world workload traces from multiple large-scale computing facilities.
Specifically, workload data are drawn from the Massachusetts Institute of Technology (MIT) Supercloud dataset~\cite{MIT-data}, the Arizona State University (ASU) Sol high-performance computing cluster, and the Blue Waters system monitoring dataset~\cite{BlueWatersData}. 
These datasets provide available CPU and GPU utilization measurements with associated power usage, covering diverse computing environments and operational regimes. 
For cooling and auxiliary subsystems, we adopt representative profiles reported by ASHRAE~\cite{ASHRAE_kaggle}, as detailed cooling measurements are typically not shared by data center operators. 

For each workload trace, the calibrated data center model with pattern-consistent calibration is compared against representative baselines from both data-driven and physics-based modeling categories. 
For data-driven approaches, we consider a statistical Hidden Markov Model (HMM)~\cite{HMM} and a learning-based recurrent neural network (RNN)~\cite{kumar2018long}. 
For physics-based modeling, we select a recent data center load model~\cite{0522jimenez2025data} as an uncalibrated baseline.

To study the interaction between data center facilities and grid dynamics, transient simulations are conducted using \texttt{ANDES}~\cite{ANDES}, which supports customized differential-algebraic load subsystems and is well suited for large-scale transient stability analysis. 
Simulations are performed mainly on the IEEE 39-bus system, where a fixed number $K$ of buses are randomly selected and assigned as data center sites. \rev{For each selected site, the protection and recovery parameters are randomly sampled from the physically admissible ranges in Table~\ref{Table:para} and then held fixed throughout the simulation, thereby representing facility-level heterogeneity in protection and reconnection behavior.}

To induce system dynamic responses, a three-phase fault is applied at a randomly selected bus at 5.0 s and cleared after 100 ms.
For reference, system responses obtained using the calibrated data center model are compared against those produced by a conventional ZIP load model, representing traditional industrial loads such as steel mills, and the uncalibrated physics-based data center model~\cite{0522jimenez2025data}. Scalability is examined in Section~\ref{sec:scalability} by extending the simulations to the NPCC 140-bus and WECC 179-bus benchmark systems.

Implementation details are as follows. Simulations are conducted using \texttt{Python~3.12} and \texttt{ANDES~v1.9.3} using Google Colab with an NVIDIA A100 GPU.
\rev{For TCL training in Eq.~\eqref{eq:TCL-loss}, the encoder is optimized using Adam with a learning rate of $10^{-3}$, a batch size of 128, and 200 epochs, with window length $L=10$, stride $H=5$, embedding dimension $d=64$, and contrastive temperature $\tau=0.1$. Two augmented views of each window are generated independently using amplitude scaling sampled from $\mathcal{U}(0.9,1.1)$ and additive zero-mean Gaussian noise, whose standard deviation for each input channel is set to 0.1 times the corresponding training-data standard deviation.
The parameter-calibration problem in Eq.~\eqref{eq:TCL-matching} is solved independently for each subsystem using bounded differential evolution with the \texttt{best1bin} strategy. The default parameter vector is included in the initial population, and all candidates are constrained to the physically admissible ranges in Table~\ref{Table:para}. The optimization terminates when the relative improvement in the best objective value remains below $10^{-4}$ for 10 consecutive generations or when 100 generations are reached.}

\begin{table}[h]
\centering
\caption{Initial values and physically admissible ranges for parameter
calibration. Each entry is reported as initial value $[\mathrm{min},\mathrm{max}]$.}
\label{Table:para}
\scriptsize
\renewcommand{\arraystretch}{1.08}
\setlength{\tabcolsep}{4pt}

\begin{minipage}{\columnwidth}
\begin{tabular}{@{}l l@{}}
\toprule
\textbf{Subsystem} & \textbf{Parameter: initial value $[\mathrm{min},\mathrm{max}]$} \\
\midrule
Workload &
\begin{tabular}[t]{@{}l@{}}
$p_{\mathrm{base}}:0.2[0.05,0.5]$, \;
$p_{\mathrm{full}}:1[0.7,1.2]$, \;
$\tau_{\eta}:30[5,120]$ \\

$\mu_{\eta}:0.5[0.05,0.95]$, \;
$\sigma_{\xi}:0.05[0.005,0.2]$, \;
$\lambda:1[0.1,5]$ \\

$\mu_{\ln A}:-2.3[-4.61,-0.69]$, \;
$\sigma_{\ln A}:0.5[0.1,1]$
\end{tabular}
\\[2pt]
\midrule

Cooling &
\begin{tabular}[t]{@{}l@{}}
$R_s:0.02[0.005,0.05]$, \;
$X_s:0.12[0.05,0.25]$, \;
$X_m:3[1.5,6]$ \\

$R_r:0.015[0.005,0.05]$, \;
$X_r:0.1[0.05,0.25]$, \;
$H_m:0.5[0.1,2]$ \\

$V_{\mathrm{stall}}:0.7[0.55,0.85]$, \;
$\tau_{\mathrm{stall}}:0.1[0.02,0.5]$, \;
$T_{\mathrm{cool}}:10[1,60]$
\end{tabular}
\\[2pt]
\midrule

Auxiliary &
\begin{tabular}[t]{@{}l@{}}
$p_{\mathrm{aux},0}:0.08[0.01,0.2]$, \;
$\alpha_Z:0.3[0,1]$, \;
$\alpha_I:0.2[0,1]$ \\

$\alpha_P:0.5[0,1]$, \;
$\beta_{\mathrm{aux}}:0.33[0,0.75]$\\
\end{tabular}
\\[2pt]
\midrule

Protection &
\begin{tabular}[t]{@{}l@{}}
$V_{\mathrm{ref}}:1$, \;
$\omega_{\mathrm{ref}}:60$, \;
$\Delta V:.3[.2,.5]$, \;
$\Delta\omega:.02[.01,.05]$ \\

$t_{\mathrm{delay}}^{\mathrm{trip}}:.05[.01,.5]$, \;
$t_{\mathrm{wait}}^{\mathrm{recon}}:1[.1,10]$, \;
$t_{\mathrm{delay}}^{\mathrm{recon}}:10[1,60]$ \\

$\kappa_{\min}:0[0,.5]$, \;
$\kappa_{\max}:1[.5,1]$, \;
$r_{\kappa}:.1[.01,1]$
\end{tabular}
\\[2pt]
\bottomrule
\end{tabular}

\vspace{1mm}
\raggedright
\footnotesize
\textit{Note:} 
The constraints $p_{\mathrm{base}}<p_{\mathrm{full}}$ and
$\alpha_Z+\alpha_I+\alpha_P=1$ are enforced during optimization.
The burst-amplitude distribution is parameterized by
$\boldsymbol{\theta}_{A}=(\mu_{\ln A},\sigma_{\ln A})$.
\end{minipage}
\end{table}

\rev{
For calibration evaluation, each trace is partitioned into 5,000 training and 2,000 testing samples, with identical splits used across all methods. The workload experiments include 100 MIT Supercloud GPU traces, 10 ASU Sol traces, and 100 Blue Waters GPU traces; the cooling and auxiliary datasets each include 100 traces. The HMM baseline uses six latent states with Gaussian emissions and diagonal covariance matrices and is trained by expectation--maximization for up to 100 iterations with a convergence tolerance of $10^{-4}$. The RNN baseline comprises two 64-unit recurrent layers followed by a fully connected output layer, uses an input window of 20 samples, and is trained with Adam using a learning rate of $10^{-3}$, a batch size of 128, and up to 100 epochs, with early stopping after 10 epochs without validation-loss improvement.
}

\rev{
During \texttt{ANDES} implementation, at each selected data center bus, the calibrated subsystem trajectories are aggregated at every simulation time as \(p_{\mathrm{load}}(t)=p_{\mathrm{work}}(t)+p_{\mathrm{cool}}(t)+p_{\mathrm{aux}}(t)\) and \(q_{\mathrm{load}}(t)=q_{\mathrm{cool}}(t)+q_{\mathrm{aux}}(t)\). Their values at $t=0$ are used to initialize a converged AC power flow, and the subsequent demands are injected into the corresponding buses through a user-defined \texttt{pert()} function. Based on the bus-voltage and grid-frequency responses solved by \texttt{ANDES}, the function evaluates the protection and recovery conditions and updates the retained-load fraction $\kappa(t)$, such that the grid-supplied demands become $p_{\text{dc}}(t)=\kappa(t)p_{\text{load}}(t)$ and $q_{\text{dc}}(t)=\kappa(t)q_{\text{load}}(t)$. 
The generator-side dynamics are directly inherited from the native ANDES case files \texttt{ieee39/ieee39\_full.xlsx}, \texttt{npcc/npcc.xlsx}, and \texttt{wecc/wecc\_full.xlsx}, which already include the corresponding synchronous-generator, excitation-system, and turbine-governor models and their unit-specific parameters. Specifically, the IEEE 39-bus case contains ten \texttt{GENROU} machines with \texttt{IEEEX1} exciters and \texttt{TGOV1N} governors; the NPCC 140-bus case contains 27 \texttt{GENROU} and 21 \texttt{GENCLS} machines with 24 \texttt{IEEEX1} exciters and 29 \texttt{TGOV1} governors; and the WECC 179-bus case contains 29 \texttt{GENROU} machines with \texttt{ESDC2A}, \texttt{EXST1}, or \texttt{ESST3A} excitation systems and \texttt{IEEEG1} governors. These native controller models and parameters are retained without modification in all simulations.
The network equations are solved using implicit trapezoidal integration with Newton--Raphson iterations, a fixed step of $\Delta t = 0.01$~s, and a simulation horizon of $T_{\text{sim}} = 40$~s.
}

\subsection{\rev{Validation of Poisson Burst-Arrival Assumption}}
\label{sec:poisson-gof}

\rev{Before calibrating the OU workload parameters, we first examine whether the Poisson-arrival assumption used in Eq.~(2) is consistent with real workload traces. This validation is performed directly on the MIT Supercloud, ASU Sol, and Blue Waters workload data. For each normalized workload trace $\hat{\eta}(t)$, we define a job-arrival event as a sufficiently large positive utilization change:
\(
\Delta \hat{\eta}(t_m) > \gamma \cdot \mathrm{MAD}(\Delta \hat{\eta}),
\)
where $\Delta \hat{\eta}(t_m)=\hat{\eta}(t_m)-\hat{\eta}(t_{m-1})$, $\mathrm{MAD}(\cdot)$ is the median absolute deviation, and adjacent detections within a short refractory interval are merged into a single arrival event. This produces an event sequence $\{t_k\}$ and the corresponding inter-arrival times $\Delta T_k=t_k-t_{k-1}$.
}

\rev{
If job arrivals follow a homogeneous Poisson process, the inter-arrival times should follow an exponential distribution. Therefore, for each workload window, we estimate $\hat{\lambda}=1/\overline{\Delta T}$ and compare the empirical distribution of $\Delta T_k$ with $\mathrm{Exp}(\hat{\lambda})$ using a bootstrap Kolmogorov--Smirnov (KS) goodness-of-fit test. In each bootstrap trial, synthetic exponential samples with the same size are generated, $\lambda$ is re-estimated, and the KS statistic is recomputed. We report the goodness-of-fit ratio
\begin{equation}
R_{\mathrm{GOF}}
=
\frac{\#{\text{windows passing the exponential inter-arrival test}}}
{\#{\text{valid workload windows}}},
\end{equation}
where a window is counted as passing if the KS-test $p$-value is larger than $0.05$. In the experiments, we use $\gamma=3$, a refractory interval of 5 samples, and require at least 20 detected arrivals for a valid workload window.
}

\begin{table}[H]
\centering
\caption{Goodness-of-fit test for the Poisson job-arrival assumption on real workload traces.}
\label{tab:poisson_gof}
\rev{
\begin{tabular}{lccc}
\toprule
Dataset & Valid windows & $R_{\mathrm{GOF}}$ & Interpretation \\
\midrule
MIT Supercloud & 48 & 81.3\% & Good fit \\
ASU Sol & 52 & 76.9\% & Good fit \\
Blue Waters & 44 & 63.6\% & Partial fit \\
\bottomrule
\end{tabular}
}
\end{table}

\rev{
As shown in Table~\ref{tab:poisson_gof}, the Poisson-arrival assumption is most consistent with the MIT Supercloud and ASU Sol workload traces, where the detected burst-arrival intervals frequently pass the exponential goodness-of-fit test. The Blue Waters trace shows a weaker but still partially acceptable fit, suggesting stronger scheduling or regime-mixing effects. These results support using Poisson arrivals as a first-order model for burst-triggering events in GPU workload traces, while also indicating that more aggregate or periodic load regimes may require the time-varying or regime-dependent extension discussed in Section~II-A.
}

\subsection{Validation of Pattern-Consistent Calibration}
\label{sec:validation-calibration}

\begin{table*}[ht]
\centering
\caption{Shape similarity metrics across real workload, cooling load, and auxiliary load datasets. All simulated traces are normalized prior to metric computation.}
\label{tab:shape-similarity-big}
\scriptsize
\setlength{\tabcolsep}{6pt}
\begin{tabular}{llccccc}
\toprule
\textbf{Metric} &
\textbf{Model} &
\textbf{MIT Load}~\cite{MIT-data} &
\textbf{ASU Load} &
\textbf{BlueWaters Load}~\cite{BlueWatersData} &
\textbf{Cooling Load} \cite{ASHRAE_kaggle} &
\textbf{Auxiliary Load} \cite{ASHRAE_kaggle} \\
\midrule

\multirow{4}{*}{DTW Distance ($\downarrow$)}
& HMM                                & 1.407 & 31.068 & 4.097 & 2.943 & 0.472 \\
& RNN                                & 2.062 & 14.902 & 4.878 & 1.904 & 0.743 \\
& Physics-based (Uncalibrated)       & 2.668 & 17.928 & 6.568 & 16.552 & 1.148 \\
& Ours (Data-calibrated Physics)     & \textbf{1.095} & \textbf{1.278} & \textbf{3.301} & \textbf{1.783} & \textbf{0.308} \\
\midrule

\multirow{4}{*}{Max Cross-Corr. ($\uparrow$)}
& HMM                                & -0.067 & 0.232 & 0.121 & 0.775 & 0.965 \\
& RNN                                & -0.302 & 0.962 & 0.214 & 0.847 & 0.961 \\
& Physics-based (Uncalibrated)       & -0.179 & 0.200 & -0.011 & 0.278 & 0.088 \\
& Ours (Data-calibrated Physics)     & \textbf{0.524} & \textbf{0.981} & \textbf{0.219} & \textbf{0.920} & \textbf{0.970} \\
\midrule

\multirow{4}{*}{Cosine Similarity ($\uparrow$)}
& HMM                                & 0.994 & 0.998 & \textbf{0.985} & 0.996 & 0.999 \\
& RNN                                & 0.828 & 0.998 & 0.580 & 0.995 & 0.997 \\
& Physics-based (Uncalibrated)       & 0.888 & 0.987 & 0.649 & 0.993 & 0.990 \\
& Ours (Data-calibrated Physics)     & \textbf{0.996} & \textbf{0.999} & 0.950 & \textbf{0.998} & \textbf{0.999} \\
\bottomrule
\end{tabular}
\end{table*}

We validate the proposed pattern-consistent calibration framework by examining its ability to reproduce realistic workload, cooling, and auxiliary load dynamics observed in real-world data. 
For fair comparison, each load trace is split into disjoint training and testing segments. 
The data-driven baselines, HMM~\cite{HMM} and RNN~\cite{kumar2018long}, are trained on the training segment and evaluated using one-step-ahead prediction on the testing segment. 
\rev{For workload validation, we additionally include an autoregressive integrated moving-average (ARIMA) model to examine whether a conventional statistical time-series model can capture stochastic workload variations. It is fitted on the same training segment and evaluated through one-step-ahead prediction over the testing segment.}
The calibrated data center model uses the same training data for parameter calibration and is then simulated forward over the testing period, while the uncalibrated physics-based baseline~\cite{0522jimenez2025data} is simulated on the testing segment using default parameters.

For workload validation, we use traces from the MIT Supercloud~\cite{MIT-data}, the ASU Sol, and the Blue Waters dataset~\cite{BlueWatersData}. 
The results are shown in Fig.~\ref{fig:ou_match}. 
The uncalibrated physics-based baseline produces overly oscillatory trajectories that fail to capture the temporal structure of empirical workloads across all three datasets, indicating unrealistic data center behavior that can mislead grid-impact assessments. 
\rev{ARIMA fits the training segment closely but carries the learned linear dependence into testing, limiting its ability to adapt to nonstationary workload regimes. For example, in the MIT dataset, it fails to reproduce the idle interval around $t\in[0,120]$ because its prediction remains influenced by the preceding training regime.}
\rev{The HMM and RNN baselines capture some local variations but tend to regress toward average behavior during testing. 
For example, although both follow the relatively steady operating phase around $t\approx100$ in the ASU Sol dataset, they fail to reproduce the subsequent sustained increase associated with a burst event.}
As a result, these models have limited ability to capture regime transitions and bursty workload dynamics. 
In contrast, the calibrated data center model does not aim for pointwise accuracy. 
Instead, its trajectories align with the dominant temporal patterns observed across all datasets while capturing burstiness and variability. 
This behavior stems from the physics-based workload formulation and indicates that the calibrated parameter set represents physically meaningful workload dynamics rather than a purely statistical fit.

\begin{figure}[h]
    \centering
    \vskip -0.1in
    \includegraphics[width=1\linewidth]{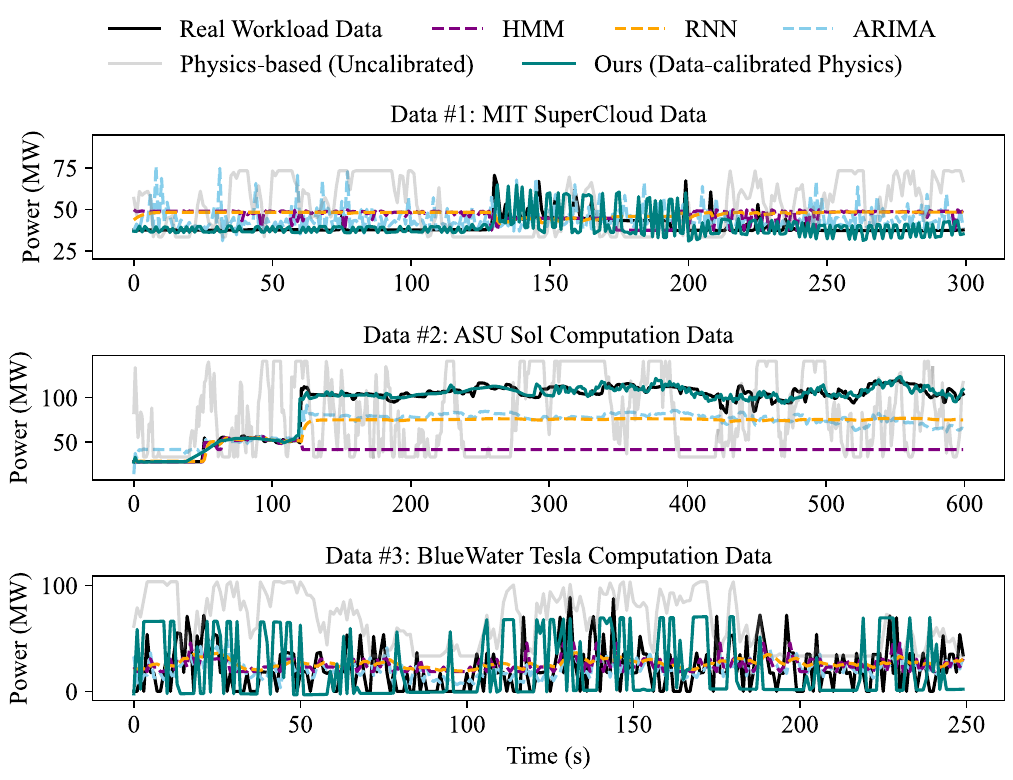}
    \vskip -0.1in
    \caption{\rev{Comparison between real CPU/GPU workload traces and simulated workload trajectories for three computing sites.}}
    \label{fig:ou_match}
    \vskip -0.1in
\end{figure}

\rev{To complement the time-domain validation, Fig.~\ref{fig:spectrogram} compares the spectrograms of four representative MIT workload traces with their calibrated counterparts. For Traces 1 and 2, the calibrated model preserves the major transitions between low- and high-activity intervals and the associated broadband spectral changes, although some abrupt boundaries are smoothed. For the more continuously varying Traces 3 and 4, it reproduces the overall distribution and persistence of spectral energy across time and frequency. The calibrated traces also exhibit moderately stronger low-frequency content, consistent with the mean-reverting workload dynamics. Overall, the results indicate that the calibrated model captures the dominant time-varying fluctuation scales of the measured workloads rather than merely matching their time-domain amplitudes.}

\begin{figure}[h]
    \centering
    \vskip -0.15in
    \includegraphics[width=1\linewidth]{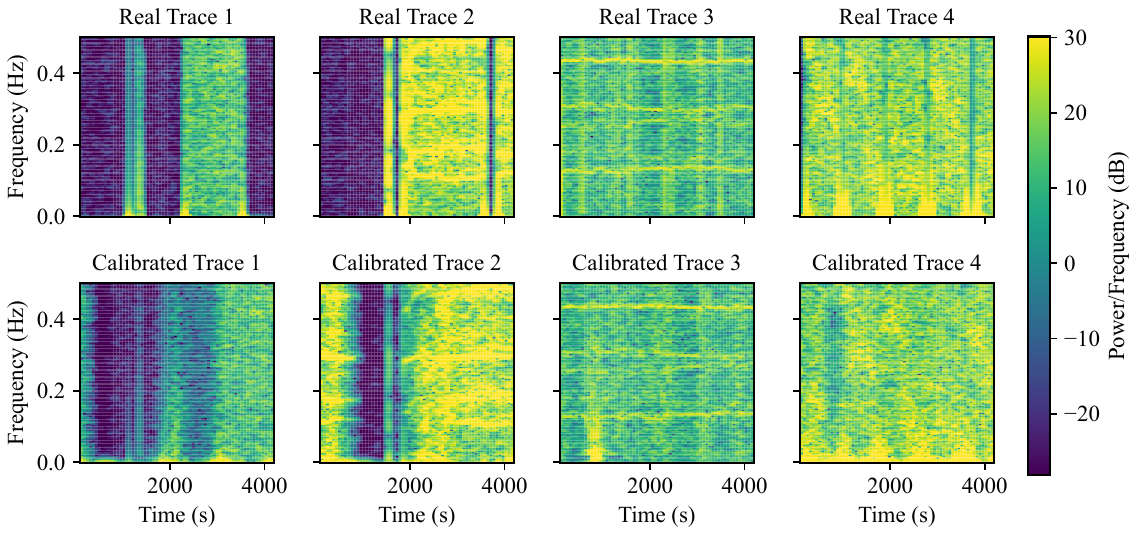}
    \vskip -0.1in
    \caption{\rev{Spectrogram comparison of four representative MIT workload traces and their calibrated counterparts.}}
    \label{fig:spectrogram}
    \vskip -0.1in
\end{figure}

For cooling and auxiliary loads, we use measured cooling-demand profiles and reconstructed auxiliary-load trajectories reported by ASHRAE~\cite{ASHRAE_kaggle}. The results are shown in Fig.~\ref{fig:cooling-pattern-match} and Fig.~\ref{fig:aux-pattern-match}. The uncalibrated cooling model and default ZIP-based auxiliary model produce smooth or nearly constant responses that fail to capture the abrupt activation and high variability observed in empirical data. Data-driven baselines, such as HMM and RNN, while capable of fitting training data, do not generalize reliably to these subsystems due to the absence of explicit physical structure and are therefore not suitable for reliable transient simulation. After applying the proposed pattern-consistent calibration, the simulated cooling trajectories capture elevated operating regimes and bursty fluctuations characteristic of motor-driven cooling systems, which are not reproduced by data-driven baselines or uncalibrated physics-based models. The calibrated auxiliary load similarly exhibits improved alignment in both amplitude and temporal variation.

\begin{figure}[h]
    \centering
    \vskip -0.15in
    \includegraphics[width=1\linewidth]{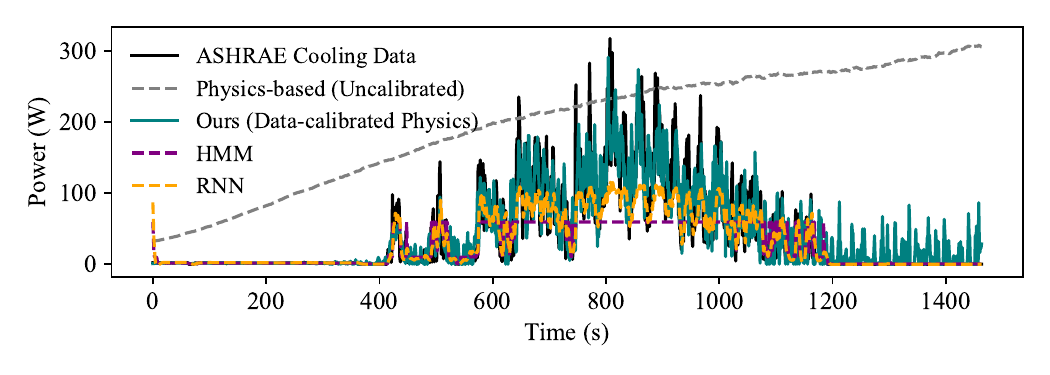}
    \vskip -0.2in
    \caption{Empirical cooling load traces and simulated cooling trajectories.}
    \label{fig:cooling-pattern-match}
    \vskip -0.1in
\end{figure}

\begin{figure}[h]
    \centering
    \vskip -0.15in
    \includegraphics[width=1\linewidth]{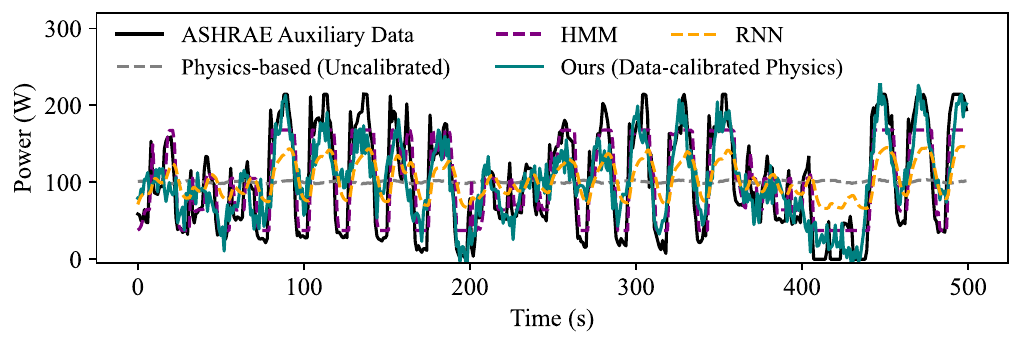}
    \vskip -0.1in
    \caption{Empirical auxiliary load traces and simulated auxiliary trajectories.}
    \label{fig:aux-pattern-match}
    \vskip -0.1in
\end{figure}

To quantify similarity between simulated and empirical load trajectories, we use three complementary shape-based metrics: dynamic time warping (DTW) distance, maximum cross-correlation, and cosine similarity. These metrics assess temporal and structural alignment rather than pointwise prediction accuracy, which is not required for system-level dynamic analysis. As summarized in Table~\ref{tab:shape-similarity-big}, the proposed calibrated physics-based model consistently achieves lower DTW distance and higher cross-correlation and cosine similarity than all baseline models across workload, cooling, and auxiliary datasets. While HMM and RNN baselines perform competitively in isolated metrics or datasets, their performance is less consistent across operating regimes and subsystems.

\subsubsection{Robustness to Parameter Initialization}

\rev{Another requirement for practical deployment is that the calibration result should not depend strongly on the initial parameter guess. Although calibration is performed locally using facility measurements, the reduced-order parameters $\{\Theta^{\ast}_{\mathrm{work}}, \Theta^{\ast}_{\mathrm{cool}}, \Theta^{\ast}_{\mathrm{aux}}\}$ are latent model quantities and are not directly known to either the facility operator or the utility. Consequently, the optimization is typically initialized using generic defaults, industry-average values, or heuristic estimates that may differ substantially from the facility-specific solution. We therefore evaluate initialization robustness on the MIT Supercloud dataset using 20 trials with randomly sampled initial parameters and compare the results with the uncalibrated physics-based baseline~\cite{0522jimenez2025data}. As shown in Fig.~\ref{fig:various_init}, the uncalibrated model exhibits substantial dispersion across initializations, whereas the proposed calibration produces considerably more consistent workload behavior. These results indicate that the calibration is relatively insensitive to inaccurate initial guesses, which is important for reliable facility-specific model estimation.}

\begin{figure}[h]
    \centering
    \vskip -0.1in
    \includegraphics[width=1\linewidth]{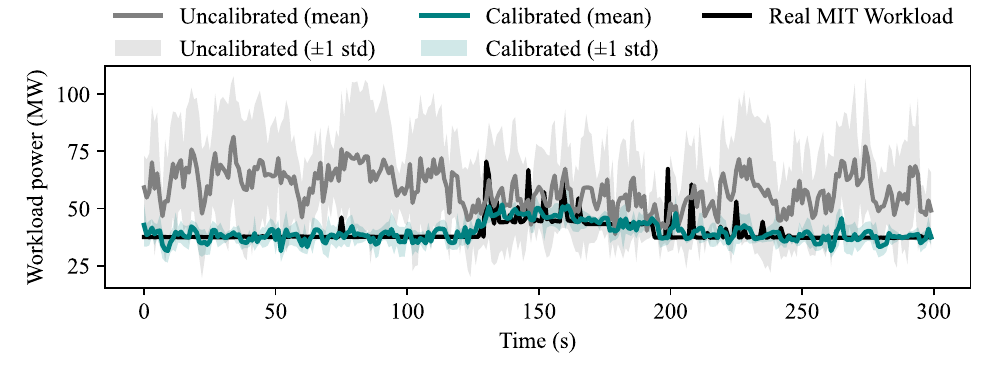}
    \vskip -0.2in
    \caption{Baseline and pattern-consistent calibrated data center workload models under varying parameter initializations.}
    \label{fig:various_init}
    \vskip -0.1in
\end{figure}

\subsection{Effect and Sensitivity of the TCL Feature Module}

The proposed calibration framework employs TCL to align real and simulated load trajectories in feature space, focusing on temporal patterns and overall magnitude rather than pointwise matching.
We evaluate the effectiveness of TCL through a series of ablation and sensitivity studies. Specifically, we examine:
(1) feature representation (TCL versus hand-crafted statistics),
(2) calibration loss (feature-wise TCL loss versus sequence-level MSE),
(3) representation scope (window-based TCL versus autoencoder),
(4) augmentation strategy (full, noise-only, scale-only),
(5) temporal window length ($L=3,5,10$), and
(6) feature dimension ($d=16,64,256$).
In Fig.~\ref{fig:ablation}, we report the TCL contrastive loss and the DTW distance between real and calibrated simulated trajectories.

\rev{In Fig.~\ref{fig:ablation}, the pure sequence-MSE and hand-crafted-feature baselines bypass the TCL encoder; therefore, their contrastive losses are not defined and are reported as ``N/A''. The hand-crafted baseline represents each trajectory using fixed statistical and temporal descriptors, including the mean, variance, and skewness of raw load trace, rather than learned window-level embeddings. Both baselines produce larger DTW distances than the full TCL configuration, indicating that pointwise alignment and fixed summary statistics are less effective at capturing the bursty and regime-dependent temporal structure of data center loads.}

The full TCL configuration with $L=5$ and $d=64$ achieves the lowest DTW distance among all evaluated methods. Calibration based on sequence-level MSE, hand-crafted features, or autoencoder representations yields substantially larger DTW errors, indicating poorer recovery of temporal patterns. The sensitivity results further show that the temporal window length is critical: overly short windows may fail to capture multi-timescale burst dynamics, whereas overly long windows may oversmooth transitions between operating regimes. Varying the feature dimension beyond the nominal setting provides limited additional benefit. Overall, these results demonstrate that TCL provides a compact and effective representation for pattern-consistent calibration while remaining robust to reasonable variations in the augmentation strategy, window length, and feature dimension.

\begin{figure}[h]
    \centering
    \vskip -0.1in
    \includegraphics[width=1\linewidth]{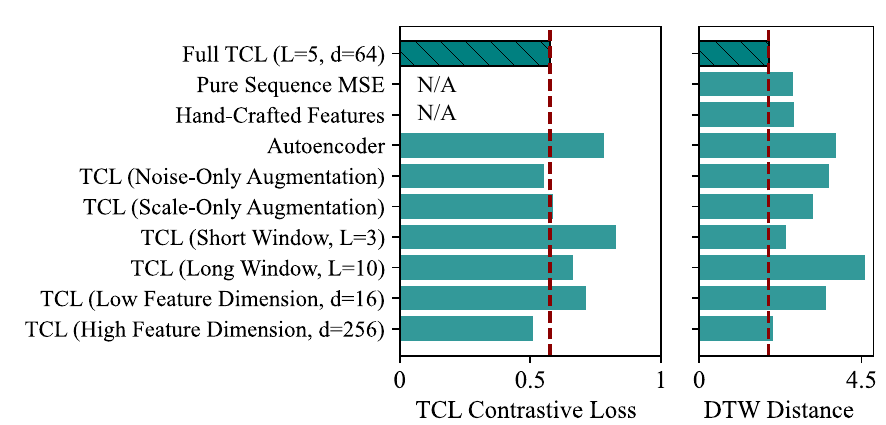}
    \vskip -0.1in
    \caption{Ablation and sensitivity analysis of the TCL-based calibration.}
    \label{fig:ablation}
    \vskip -0.15in
\end{figure}

\begin{figure*}
    \centering
    \vskip -0.1in
    \includegraphics[width=.9\linewidth]{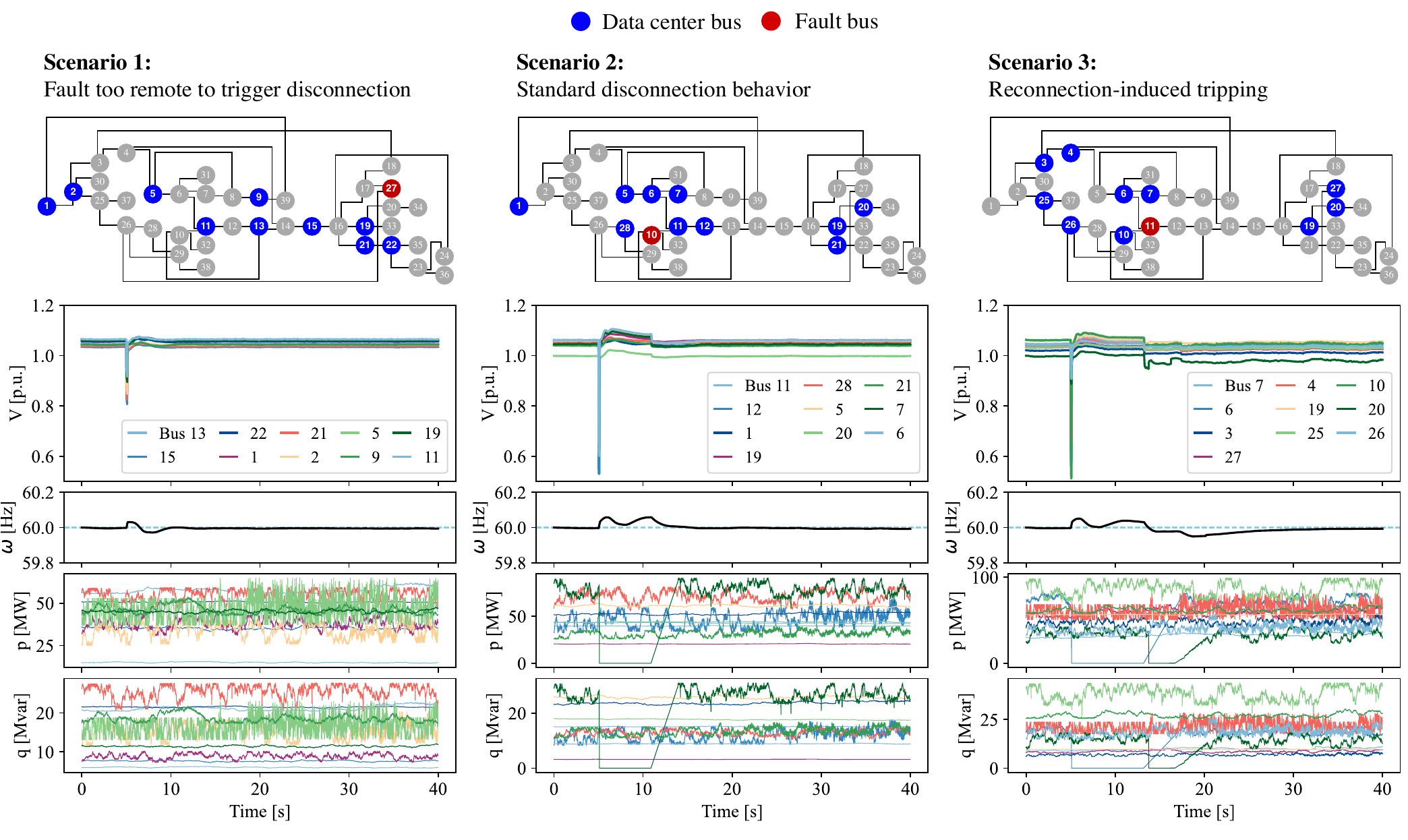}
    \vskip -0.1in
    \caption{\rev{Fault-induced transient responses in the IEEE 39-bus system with $K=10$ randomly placed data centers.}}
    \label{fig:4scenarios}
    \vskip -0.1in
\end{figure*}
\subsection{System-Level Dynamics with Heterogeneous Data Centers}
\label{sec:heterogeneous}
After validating the accuracy of the proposed load modeling approach, we integrate the calibrated load models into \texttt{ANDES}~\cite{ANDES} to evaluate system-level dynamic responses with heterogeneous data centers connected to the grid. 
Interactions among heterogeneous data centers can significantly influence grid dynamics, yet remain largely underexplored in existing studies.
We consider the IEEE 39-bus system and randomly select $K=10$ buses as data center sites to reflect their increasing penetration.
Each selected bus is assigned a data center model with slightly perturbed parameters randomly sampled from predefined ranges.
Across these experiments, we observe three qualitatively distinct system response regimes with increasing severity, which are summarized in Fig.~\ref{fig:4scenarios}.

\rev{In \emph{Scenario~1}, the data center buses are electrically distant from the simulated fault, resulting in only modest voltage sags. Their $P$ and $Q$ trajectories remain nonzero throughout the disturbance, confirming that all data centers ride through the event without activating their protection mechanisms. 
In \emph{Scenario~2}, the voltage sag at bus~7 exceeds its protection threshold, causing the corresponding data center to disconnect from the grid, as directly indicated by both its active and reactive powers dropping to zero. The resulting loss of demand causes the system frequency to rise above nominal and subsequently return toward its steady-state value. After the recovery criteria and waiting period are satisfied, the $P$ and $Q$ demands at bus~7 are restored gradually according to the prescribed reconnection ramp. This behavior closely resembles utility-reported disturbance events involving data centers~\cite{ERCOT_LEL_RideThrough_2025,shamseldein2025liability}. 
In \emph{Scenario~3}, a more critical interaction emerges. Following its initial disconnection, the data center at bus~6 reconnects through a gradual restoration of its active and reactive power demand. The resulting increase in grid-side demand produces a secondary voltage depression across the monitored buses, which subsequently triggers the protection-driven disconnection of bus~20, modeled as a comparatively voltage-sensitive data center. This demonstrates reconnection-induced tripping, where restoration of one data center disconnects another and compounds post-disturbance stress.}

Such a spectrum of transient behaviors is difficult to capture with uncalibrated or homogeneous load representations.
The calibrated modeling framework reveals how heterogeneous data centers can fundamentally reshape post-disturbance recovery and amplify localized faults into system-level events.

\subsection{\rev{Impact of Load Representation and Calibration}}
\label{sec:baseline}

\rev{
To determine whether facility-specific calibration translates into more accurate utility-facing transient assessments, we compare three data center load representations against a trace-driven reference under identical IEEE 39-bus contingencies: (1) a protected ZIP model, (2) an uncalibrated physics-based model~\cite{0522jimenez2025data}, and (3) the proposed calibrated physics-based model. The reference case directly replays workload trajectories from the MIT Supercloud dataset together with cooling and auxiliary load trajectories from the ASHRAE dataset in the ANDES simulation.
Across all methods, the data center locations, fault locations, initial operating conditions, and protection and recovery settings are held fixed. In particular, the ZIP model is augmented with the same data-center-specific protection logic used by the other models.
}

\rev{
System-level responses are evaluated using three metrics: (1) \emph{voltage nadir}, defined as the minimum post-fault bus voltage; (2) \emph{frequency overshoot}, defined as the maximum frequency deviation during recovery; and (3) \emph{reconnection delay}, defined as the interval between fault clearing and successful data center reconnection. Across 100 matched trials, we compute the absolute errors in these three metrics relative to the real-world trace reference. By holding data center placement and fault location fixed across models, this paired comparison isolates errors attributable to the load representation itself.
}

\begin{figure}[H]
    \centering
    \vskip -0.1in
    \includegraphics[width=1\linewidth]{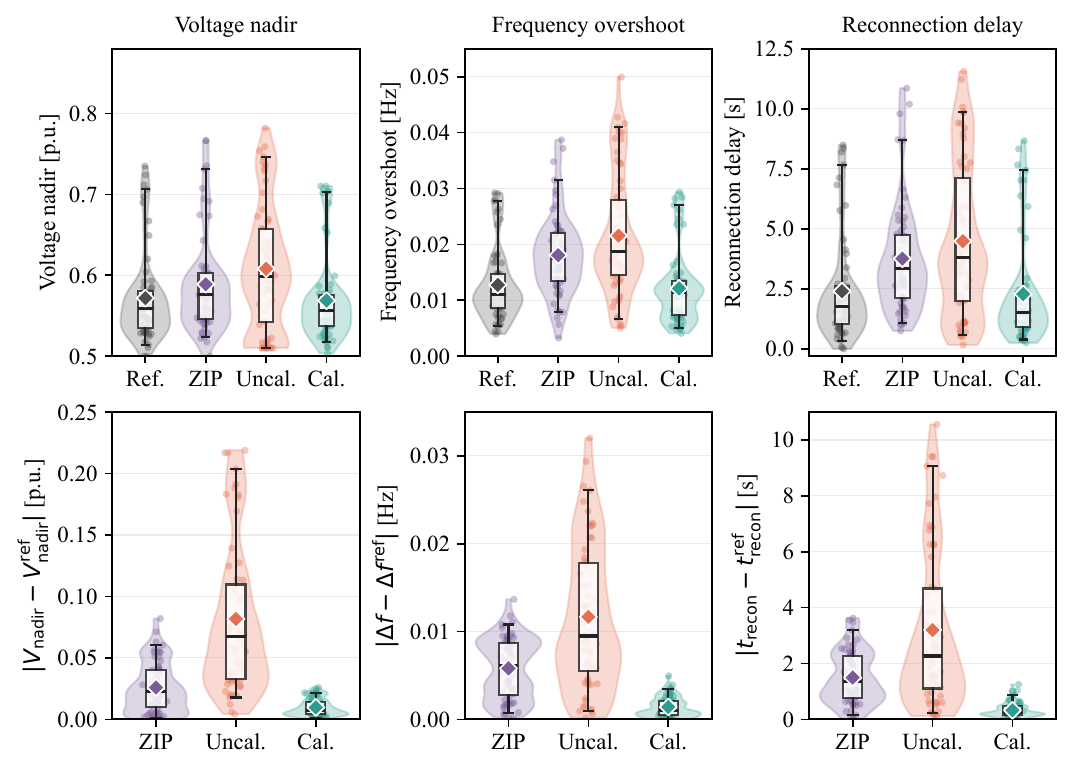}
    \vskip -0.1in
    \caption{\rev{Comparison of the real-world load reference, protected ZIP model, uncalibrated physics-based model, and proposed calibrated model. }}
    \label{fig:model_comparison}
    \vskip -0.1in
\end{figure}

\rev{
The upper row of Fig.~\ref{fig:model_comparison} compares the raw response distributions. The proposed calibrated model most closely reproduces the voltage-nadir, frequency-overshoot, and reconnection-delay distributions of the real-world load reference. The protected ZIP model captures disconnection and reconnection through the common protection layer, but its static voltage-dependent representation cannot reproduce facility-specific workload, cooling, and reactive-power dynamics. The uncalibrated physics-based model captures these dynamic components, but its limited parameterization hinders facility-specific adaptation; consequently, default or randomly assigned parameters yield broader, systematically shifted responses and poorer agreement with real data center behavior. The lower row reports the corresponding paired absolute errors. The calibrated model yields the smallest and most concentrated errors for all three metrics, whereas the ZIP baseline exhibits larger deviations and the uncalibrated model produces the widest error ranges and most severe outliers. These results show that the proposed calibration enables utilities to anticipate a more credible range of post-contingency outcomes, supporting better-informed contingency screening and protection coordination.
}

\subsection{Scalability to Large-Scale Power Systems}
\label{sec:scalability}

\rev{
To evaluate scalability, we extend the simulations to the NPCC 140-bus and WECC 179-bus systems. As shown in Fig.~\ref{fig:scalability}, the proposed framework transfers directly to both networks without reformulation while preserving calibrated $P$--$Q$ variability, protection-driven disconnection, and staged restoration. When local thresholds are violated, the affected data centers exhibit simultaneous reductions of $P$ and $Q$ to zero, directly indicating grid-side disconnection. These results demonstrate that the model retains facility-level dynamics and protection behavior in large-scale transient studies, providing utilities with a scalable tool for assessing data-center-driven contingency outcomes.
}

\begin{figure}[h]
    \centering
    \vskip -0.1in
    \includegraphics[width=1\linewidth]{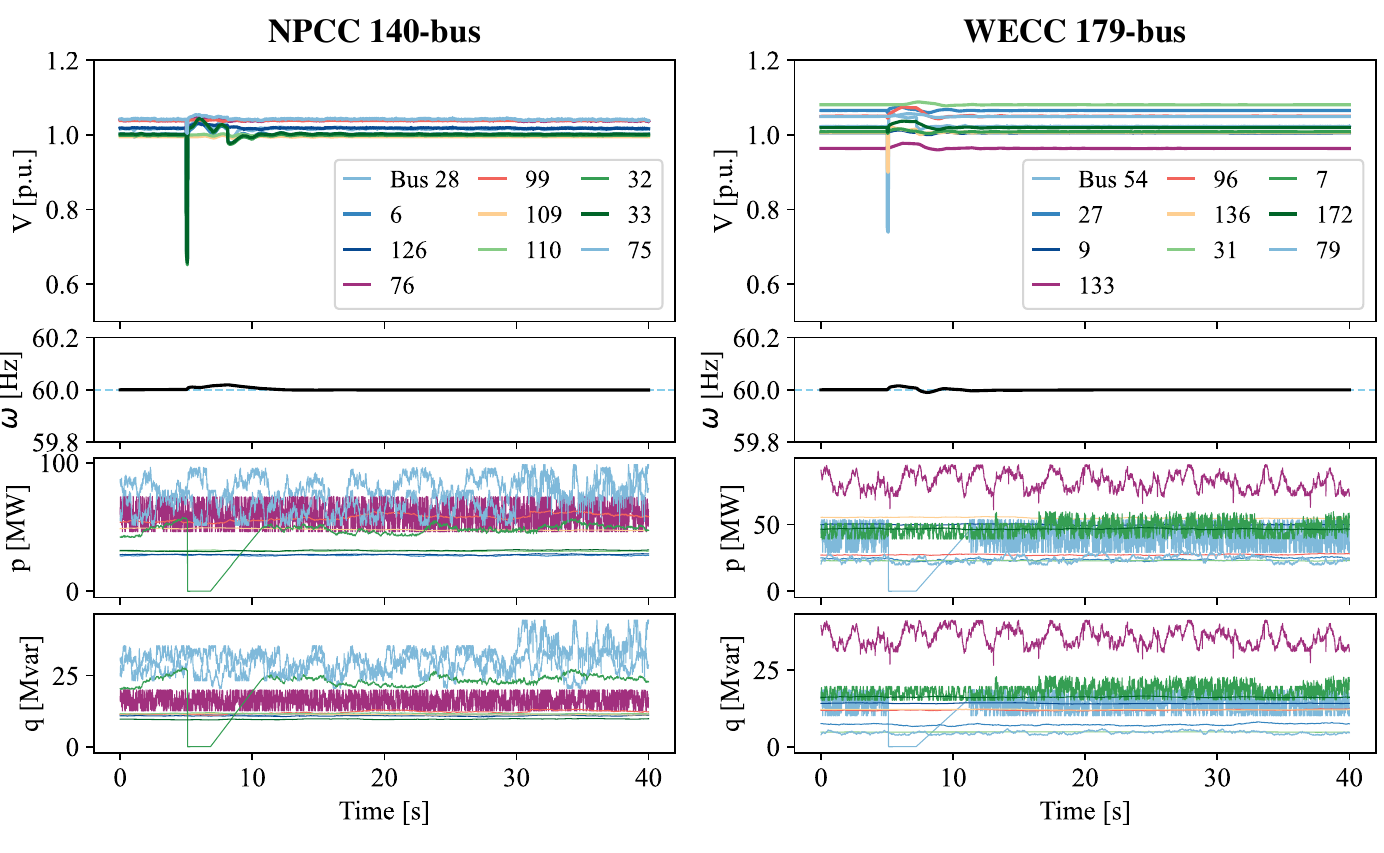}
    \vskip -0.1in
    \caption{\rev{Fault-induced transient responses in the NPCC 140-bus and WECC 179-bus systems with $K=10$ randomly placed data centers.}}
    \label{fig:scalability}
    \vskip -0.1in
\end{figure}

\section{Conclusion}
\label{sec:conclusion}
This paper proposes a calibrated dynamic load modeling framework for data centers. It is built on a physics-based backbone capturing workload variability, cooling and auxiliary dynamics, and protection-driven disconnection and reconnection behavior relevant to grid stability. Building on this structure, we develop a pattern-consistent calibration approach using temporal contrastive learning (TCL) to align model parameters with real operational data through representative temporal patterns rather than exact trajectories. 
The framework targets realistic, physics-consistent data center simulation rather than load forecasting, \rev{reduces the need to exchange full-resolution facility operational traces through local calibration}, and supports system-level dynamic studies. Experimental results show that the calibrated model reproduces real-world load behavior, and ANDES-based simulations demonstrate disturbance and recovery responses consistent with utility-reported events across systems of varying scales and data center penetration levels. 
Future work will extend this framework toward proactive control strategies for mitigating data center-induced operational risks.

\bibliographystyle{IEEEtran}
\bibliography{ref}

\end{document}